\documentclass[accepted]{arxiv}

\usepackage[american]{babel}

\usepackage{natbib} % has a nice set of citation styles and commands
\usepackage{mathtools} % amsmath with fixes and additions
\usepackage{siunitx} % for proper typesetting of numbers and units
\usepackage{booktabs} % commands to create good-looking tables
\usepackage{amssymb}
\usepackage{amsfonts}
\usepackage{amsthm}
\usepackage[ruled,linesnumbered,vlined]{algorithm2e}
\usepackage{subfig}
\usepackage{comment}
\usepackage{bm}
\usepackage{multirow}
\usepackage{fancyvrb}

\DeclareMathOperator*{\argmin}{argmin}

\SetCommentSty{mycommfont}

\makeatletter
\patchcmd{\@algocf@start}% <cmd>
  {-1.5em}% <search>
  {0pt}% <replace>
{}{}% <success><failure>
\makeatother

\theoremstyle{plain}
\newtheorem{thm}{Theorem}
\newtheorem{prop}{Proposition}
\newtheorem*{problem}{Problem}

\title{Thresholding Data Shapley for Data Cleansing Using Multi-Armed Bandits}

\author[1]{\href{mailto:<hiroyuki.namba.dx@hitachi.com>?Subject=Your paper}{Hiroyuki~Namba}{}}
\author[1]{Shota Horiguchi}
\author[1]{Masaki Hamamoto}
\author[1]{Masashi Egi}
\affil[1]{%
    Research and Development Group\\
    Hitachi, Ltd.\\
    Kokubunji, Tokyo, Japan.
}
  
\begin{document}

\maketitle

\begin{abstract}
Data cleansing aims to improve model performance by removing a set of harmful instances from the training dataset.
Data Shapley is a common theoretically guaranteed method to evaluate the contribution of each instance to model performance; however, it requires training on all subsets of the training data, which is computationally expensive.
In this paper, we propose an iterative method to fast identify a subset of instances with low data Shapley values by using the thresholding bandit algorithm.
We provide a theoretical guarantee that the proposed method can accurately select harmful instances if a sufficiently large number of iterations is conducted.
Empirical evaluation using various models and datasets demonstrated that the proposed method efficiently improved the computational speed while maintaining the model performance.
\end{abstract}

\section{Introduction}\label{sec:intro}
Developing accurate prediction models is one of the important goals of machine learning.
One typical approach is the model-centric approach, i.e., focusing on improving the prediction model itself and the learning algorithms.
Especially in machine learning based on deep learning, many large and complex models have been proposed, assuming the existence of large-scale labeled data \citep{imagenet,eurlex57k,gigaspeech}.
However, when such large-scale data is not available or when labels and features contain a lot of noise, it becomes difficult to obtain a well-generalized model by overfitting \citep{dcsv2}.
In such cases, the data-centric approach, which improves data used to train the model, is effective \citep{dcsv,dcs,dch}.
One specific policy of the data-centric approach is to remove harmful training instances that degrade the performance of models.
Based on this policy, this study tackles the following problem:
\begin{problem}[Data Cleansing]
    Find a subset of the training instances such that the trained model
    obtained after removing the subset has better prediction performance.
    \label{problem:data_cleansing}
\end{problem}

Several methods can be utilized to tackle data cleansing.
Some methods only consider the change in the model performance when one instance is removed, and do not consider the effect of removing multiple instances simultaneously \citep{loo,if,sgdif}. 
Therefore, even if a pair of instances identified as harmful is removed, the model performance may not improve~\citep{ifsv}.
On the other hand, data Shapley~\citep{ds} makes it possible to estimate the model performance when multiple instances are removed simultaneously.
One drawback of data Shapley is that it incurs extensive computational cost since obtaining the data Shapley values requires calculations for all instance combinations.
However, when solving data cleansing, it is not necessary to calculate the data Shapley values for all instances exactly, but it is sufficient to identify instances with somewhat low data Shapley values.

For fast and accurate data cleansing, this paper proposes thresholding data Shapley, which enables fast identification of instances with low data Shapley values.
The method uses the Anytime Parameter-free Thresholding (APT) algorithm~\citep{apt} for thresholding bandit problems.
The goal of the problems is to determine the arms that satisfy a condition as fast as possible in iterations each of which provides partial information on the selected arm.
We utilize this to determine instances with low data Shapley values fast by repeating iterations of obtaining partial information on the data Shapley values of selected instances.
Furthermore, to make the proposed method faster, we propose three enhancements: restriction on the cardinality of training subsets, evaluation of multiple instances at a single iteration, and pre-training using the entire training set.

The contributions of this work are summarized as follows:
\begin{itemize}
\item We propose a method to improve the inference time of low data Shapley instances using a thresholding bandit algorithm and provide three enhancements to make the proposed method more accurate and fast (Section~\ref{sec:prop}).
\item We provide a theoretical guarantee that the proposed method can accurately select harmful instances if a sufficiently large number of computational iterations is conducted (Section~\ref{sec:theo}).
\item We empirically demonstrate the effectiveness of the proposed method using various classification/regression models and datasets (Section~\ref{sec:exp}).\end{itemize}

\section{Preliminaries}\label{sec:pre}

\subsection{Data Shapley}
Let $D=\{d_1,\dots,d_N\}$ be a training set with independent and identically distributed observations, where $N$ denotes the number of training instances and $d_n\coloneqq\left(\bm{x}_n, y_n\right)$ is a pair of a feature vector $\bm{x}_n$ and the corresponding label $y_n$.
For example, $y_n\in\{0,1\}$ when the task is binary classification and $y_n\in\mathbb{R}$ when regression.
The purpose of data cleansing is to find a subset of $D$ that maximizes the performance of the model that predicts a label from a feature vector.
The performance metric can be the accuracy for classification problems and the negative of mean absolute error (MAE) or mean squared error (MSE) for regression problems\footnote{Since MAE and MSE are lower-better performance metrics, we take the negative of them to make them higher-better.}.
For arbitrary $D'\subseteq D$, let $V(D')\in\mathbb{R}$ be the value of a higher-better performance metric of the model trained using $D'$.

A common approach to find the best subset is to evaluate how much each instance contributes to performance improvement and remove instances with low contributions.
There are various definitions and calculation methods of the contribution of each instance~\citep{loo,if,ds}, among which data Shapley is effective from both theoretical and empirical perspectives \citep{ds}.
The definition of data Shapley is as follows.
Let $S_D$ be the set of all permutations of the set $D$.
Note that its cardinality $|S_D|$ is equal to $|D|!$.
For each $\sigma\in S_D$, let $D'(\sigma,n)$ be the set of all the instances placed before $d_n$ in $\sigma$.
For example, $N=5,n=3,\sigma=(d_2,d_4,d_1,d_3,d_5)$ results in $D'(\sigma,n)=\{d_1,d_2,d_4\}$.
The data Shapley value $\phi_n$ of the $n$-th instance $d_n$ is defined as
\begin{equation}
\phi_n=\dfrac{1}{|S_D|}\sum_{\sigma\in S_D}\{V(D'(\sigma,n)\cup \{d_n\})-V(D'(\sigma,n))\}.\label{eq:data_shapley}
\end{equation}
In short, the data Shapley value $\phi_n$ is the average improvement of $V$ by adding the instance $d_n$ to the training set $D'$ over various $D'$.

\subsection{Thresholding Bandits}
Multi-armed bandit problem is a problem in which a decision maker sequentially selects one of multiple fixed choices called arms.
Let $\mathcal{N}\coloneqq\{1,\dots,N\}$ be the set of arms and $\mathcal{T}\coloneqq\{1,\dots,T\}$ be the set of iterations.
For all $n\in\mathcal{N}$, the $n$-th arm is associated with a fixed unknown distribution $P_n$ whose expected value $\mu_n$ is also unknown, and one obtains a random reward on the basis of the distribution when selecting the $n$-th arm.
At each iteration $t\in \mathcal{T}$, the decision maker selects an arm to achieve a goal on the basis of the historical rewards $r_1,\dots,r_{t-1}$. 
In this study, we consider the thresholding bandit problem, i.e., the goal is to identify all the arms whose expected values are not larger than a predetermined threshold $\tau$ at the end of iteration $T$.

For the thresholding bandit problem, the APT algorithm is proposed with a theoretical guarantee \citep{apt}.
At each iteration $t\in \mathcal{T}$, the algorithm calculates
\begin{equation}\label{equ:bn}
B_n(t)=\sqrt{T_n(t)}\{|\hat{\mu}_n(t)-\tau|+\varepsilon\}
\end{equation}
for all $n\in\mathcal{N}$ and selects the arm $n$ with the minimum $B_n(t)$.
Here, $T_n(t)$ is the number of iterations at which the arm $n$ was selected before the start of iteration $t$; thus, the smaller $T_n(t)$ leads to a more additional search for the $n$-th arm.
$\hat{\mu}_n(t)$ is the empirical mean of the reward from the $n$-th arm before the start of iteration $t$; the closer $\hat{\mu}_n(t)$ is to the threshold $\tau$, the more frequently the $n$-th arm is selected.
$\varepsilon\geq 0$ is a parameter to control the trade-off between accuracy of identification and required computation time.
The APT algorithm finally outputs the set of all the arms whose empirical reward is not larger than the threshold, specifically, $\left\{n\mid\hat{\mu_n}(T)\leq\tau\right\}\eqqcolon \hat{\mathcal{N}}_\text{harmful}$. 

The APT algorithm has a theoretical guarantee, that is, the probability of failure $p_{\mathrm{fail}}$ can be upper bounded when the number of iterations $T$ is sufficiently large.
Here, ``failure'' means that at least one arm $n$ exists such that $\mu_n\leq \tau-\varepsilon$ while $n\notin \hat{\mathcal{N}}_\text{harmful}$ or $\mu_n\geq\tau+\varepsilon$ while $n\in \hat{\mathcal{N}}_\text{harmful}$.
Before explaining the theorem itself, we introduce two concepts used in the theorem.
One is the measure of problem complexity $H$ defined as follows:
\begin{equation}\label{eq:complexity}
H=\sum_{n\in\mathcal{N}}({|\mu_n-\tau|+\varepsilon})^{-2}.
\end{equation}
The other is the sub-Gaussian distribution. For a positive real number $R$, the probability distribution $P$ is called $R$-sub-Gaussian if the random variable $X$ following $P$ satisfies
\begin{equation*}
E[\exp(tX-tE[X])]\leq\exp\left(\dfrac{R^2t^2}{2}\right),
\end{equation*}
for any real number $t$.
Based on these concepts, we describe the theoretical guarantee of the APT algorithm stated in Section 2.3 in~\citet{apt}.

\begin{thm}\label{the:1}
Assume that $T\geq 256R^2H\log(N(1+\log T))$ and the reward distribution $P_n$ is $R$-sub-Gaussian for all $n\in\mathcal{N}$.
For the APT algorithm, it holds that
\begin{equation*}
p_{\mathrm{fail}}\leq\exp\left(-\dfrac{T}{128R^2H}\right).
\end{equation*}
\end{thm}

If there are many arms with their expected rewards close to~$\tau$, the complexity $H$ becomes large and thus the upper bound of the probability of failure also becomes large.

\section{Proposed Method}\label{sec:prop}
One drawback of data Shapley is that it requires model training for all the possible subsets of the training set as in Equation (\ref{eq:data_shapley}), which leads to a large computational cost.
Although there exist fast methods for specific training algorithms~\citep{knnshap,tknnshap} and approximation methods~\citep{ds,fastds} for data Shapley values, it still takes quite a long time for general training algorithms due to a large amount of model retraining.
However, to only solve data cleansing, it is not necessary to compute $\phi_n$ exactly for all instances $d_n$; it is sufficient to identify those with small $\phi_n$.
For this purpose, we propose a fast method to identify instances with low data Shapley values named thresholding data Shapley (TDShap).

\subsection{Formulation as Thresholding Multi-armed Bandit Problem}
The data Shapley value of the $n$-th instance defined in Equation~(\ref{eq:data_shapley}) can be viewed as the average of
\begin{equation}\label{equ:Phi}
\Phi_n(\sigma)=V(D'(\sigma,n)\cup \{d_n\})-V(D'(\sigma,n))
\end{equation}
over all $\sigma\in S_D$.
Hence, $\Phi_n(\sigma)$ can be regarded as partial information of $\phi_n$.
Let us consider the following thresholding multi-armed bandit problem:
\begin{itemize}
    \item $N$ arms correspond one-to-one with the instances in the training set $D$.
    \item The reward distribution $P_n$ for arm $n\in\mathcal{N}$ is defined as the distribution of value $\Phi_n(\sigma)$ with probability $1/|S_D|$ for each $\sigma\in S_D$.
\end{itemize}
Note that the expected value of the distribution $\mu_n$ is equal to the data Shapley value $\phi_n$.
Therefore, the goal of this problem can be interpreted as identifying the arms, i.e., instances, with their data Shapley values not larger than the predetermined threshold $\tau$.
This meets the requirement of data cleansing, which is to identify all instances $n$ whose data Shapley values $\phi_n$ are not larger than the threshold $\tau$ without naively computing the values.

Since we have formulated data cleansing as a thresholding multi-armed bandit problem, the APT algorithm can be a solution for it.
The APT algorithm enables the selection of instances whose data Shapley values are not larger than $\tau$ through a limited number of adaptive trials.

\subsection{Basic algorithm}\label{sec:detail}
Algorithm \ref{algo:prop} shows the algorithm of TDShap.
It takes the following as inputs: the training set $D=\{d_1,\dots,d_N\}$, the validation set $D_{\text{val}}$, performance metric $V$, learning algorithm $\mathcal{A}$, and the hyperparameters $\tau$ for thresholding Shapley values and $\varepsilon$ for the precision of the threshold.
The outputs are the estimated data Shapley values for each instance in $D$ and the set of indices of expectedly harmful instances.

\begin{algorithm}[t]
    \DontPrintSemicolon
    \SetKwInOut{Input}{Input}
    \SetKwInOut{Output}{Output}
    \SetKwComment{Comment}{$\triangleright$\ }{}
    \Input{Training set $D=\{d_1,\dots,d_N\}$, validation set $D_{\text{val}}$, performance metric $V$, learning algorithm $\mathcal{A}$, hyperparameters $\tau,\varepsilon$}
    \Output{Estimated data Shapley values $\hat{\phi}_1,\dots,\hat{\phi}_N$, harmful instances $\hat{\mathcal{N}}_\text{harmful}=\{n\mid\hat{\phi}_n\leq \tau\}$}
    \BlankLine
    \For(\Comment*[f]{Initialization}){$n\gets1$ \KwTo $N$\label{algline:init}}{
        $\sigma\overset{\mathrm{iid}}{\sim}\text{Uniform}\left(S_D\right)$\Comment*[r]{Random permutation of $D$}\label{algline:sigma_0}
        Calculate $\Phi_n(\sigma)$ \Comment*[r]{Using Equation (\ref{equ:Phi})}\label{algline:init_phi1}
        $\hat{\phi}_n\gets\Phi_n(\sigma)$\label{algline:init_phi2}\;
        $T_{n}\gets1$\label{algline:init_tn}\;
    }
    \While{stopping criterion not met}{\label{algline:while}
        $n\overset{\mathrm{iid}}{\sim}\text{Uniform}\left(\argmin_n B_n\right)$\Comment*[r]{$B_n$ from Equation~(\ref{equ:bn})}\label{algline:sample_n}
        $\sigma\overset{\mathrm{iid}}{\sim}\text{Uniform}\left(S_D\right)$\Comment*[r]{Random permutation of $D$}\label{algline:sample_sigma}
        Calculate $\Phi_{n}(\sigma)$ \Comment*[r]{Using Equation (\ref{equ:Phi})}\label{algline:calc_phi}
        $\hat{\phi}_{n}\leftarrow\dfrac{T_{n}}{T_{n}+1}\hat{\phi}_{n}+\dfrac{1}{T_{n}+1}\Phi_{n}(\sigma)$\;
        $T_{n}\leftarrow T_{n}+1$\;\label{algline:update_tn}
    }
    \caption{Thresholding data Shapley}\label{algo:prop}
\end{algorithm}

The first block (lines \ref{algline:init}--\ref{algline:init_tn}) initialize the data Shapley values for each instance.
Each initial value is calculated by using a random permutation of the training set $\sigma$, which is uniformly sampled from $S_D$ (line~\ref{algline:sigma_0}).
For each $n\in\mathcal{N}$, two prediction models are trained using $D'(\sigma,n)\cup\{d_{n}\}$ and $D'(\sigma,n)$, respectively.
The models are used to calculate $\Phi_n\left(\sigma\right)$ using Equation (\ref{equ:Phi}), which is the initial estimation of the data Shapley value $\hat{\phi}_n$ (lines \ref{algline:init_phi1}--\ref{algline:init_phi2}).
$T_n$, which denotes the number of updates of the data Shapley value of the $n$-th instance, is initialized with one (line~\ref{algline:init_tn}).

The second block (lines \ref{algline:while}--\ref{algline:update_tn}) updates the estimation of data Shapley iteratively.
The instance $n$ whose data Shapley value will be updated is randomly selected from which the value of $B_n$ defined in Equation (\ref{equ:bn}) takes a minimum value (line~\ref{algline:sample_n}).
$\sigma$ here is also a random permutation of $S_D$ as in the initialization step (line~\ref{algline:sample_sigma}).
With the selected instance $n$ and permutation $\sigma$, $\Phi_n\left(\sigma\right)$ is calculated by using Equation (\ref{equ:Phi}) and the estimation of the data Shapley value $\hat{\phi}_n$ is updated as a moving average (lines \ref{algline:calc_phi}--\ref{algline:update_tn}).
The stopping criterion at line~\ref{algline:while} can be, for example, based on the computation time, number of iterations, or amount of update of $\phi_n$.

\subsection{Further enhancements}\label{sec:technique}

Algorithm \ref{algo:prop} improves data cleansing speed over using naive data Shapley.
In this subsection, we provide three enhancements to the algorithm to make it more accurate and faster.

\begin{enumerate}[label=(\roman*)]
\item \textbf{Restriction on cardinality of training subsets:} \label{enh1}
At lines \ref{algline:init_phi1} and \ref{algline:calc_phi}, if the number of instances in the training set $|D'(\sigma,n)|$ is too small,
the model may overfit $D'$ and $|\Phi_n|$ may become excessively large.
To prevent this problem, we propose restricting $\sigma$ so that $|D'(\sigma,n)|$ is not smaller than a certain predetermined value $N_{\min}$.

\item \textbf{Evaluation of multiple instances at a single iteration:} \label{enh2}
While only an instance $n$ is evaluated at each iteration of lines \ref{algline:while}--\ref{algline:update_tn}, we propose evaluating multiple instances simultaneously for faster calculation.
The detailed procedures are as follows.
We first select $K$ instances $n_{1},\dots,n_{K}$ with the smallest $B_n(t)$ values and update $\hat{\phi}_{n_{k}}$ for each.
This reduces the number of iterations required for convergence while increasing the computational cost of a single iteration.
Here, we restrict $\sigma$ sampled on line~\ref{algline:sample_sigma} on only permutations such that $n_{1},\dots,n_{K}$ occur successively in $\sigma$; this reduces
the number of training times when calculating Equation (\ref{equ:Phi}) on line~\ref{algline:calc_phi} from $2K$ to $K+1$ times and thus reduces the computation time to convergence.
For example, in the case when $K=3$ and $\sigma=(d_1,d_{n_{1}},d_{n_{2}},d_{n_{3}},d_2,d_3)$, each of $\Phi_{n_1}(\sigma)$, $\Phi_{n_2}(\sigma)$, and $\Phi_{n_3}(\sigma)$ can be evaluated by the following equations:
\begin{alignat*}{2}
\Phi_{n_1}&=V(\{d_1,d_{n_{1}}\})&&-V(\{d_1\}),\\
\Phi_{n_2}&=V(\{d_1,d_{n_{1}},d_{n_2}\})&&-V(\{d_1,d_{n_1}\}),\\
\Phi_{n_3}&=V(\{d_1,d_{n_{1}},d_{n_2},d_{n_3}\})&&-V(\{d_1,d_{n_1},d_{n_2}\}).
\end{alignat*}
Since $V(\{d_1,d_{n_{1}}\})$ and $V(\{d_1,d_{n_{1}},d_{n_2}\})$ appear twice above, only $K+1(=4)$ training is required. 
This enhancement can also be applied to the first block (lines \ref{algline:sigma_0}--\ref{algline:init_tn}).

\item \textbf{Pre-training using the entire training set:} \label{enh3}
To evaluate $\Phi_n\left(\sigma\right)$ at lines \ref{algline:init_phi1} and \ref{algline:calc_phi}, a lot of models need to be trained for various training sets $D'$.
Especially in the case of deep learning, it takes extensive time to complete training.
To reduce the number of training epochs to reduce training time, we propose using the model pre-trained using the entire training set to initialize parameters before starting the training using a subset $D'$.
Note that in this method, the influence of instances not in $D'$ may affect $V(D')$ through the pre-trained model; the effect of this on model performance will be empirically evaluated in Section \ref{sec:exp2}.

\end{enumerate}

Note that the enhancements \ref{enh1} and \ref{enh2} cause a slight difference in reward distribution; $\sigma$ sampled to evaluate $\Phi_n(\sigma)$ becomes out of uniform distribution on $S_D$ for a given instance $n$.
For example on enhancement~\ref{enh2}, when $\phi_{n_1}\simeq\tau$ holds and $\phi_{n_2}$ is extremely large, the number of evaluations of $d_{n_1}$ is much larger than that of $d_{n_2}$ and hence $\sigma$ in which $d_{n_1}$ and $d_{n_2}$ are consecutive is less likely to be selected.
This causes the expected value of $\Phi_n(\sigma)$ to not equal $\phi_n$ strictly.
Even in this case, the expected value can be regarded as a contribution metric for $V$ since $\Phi_n(\sigma)$ still expresses a weighted average value of improvements of $V$ by adding $d_n$ for various subsets.
The effect of the enhancements \ref{enh1} and \ref{enh2} will be empirically evaluated in Section~\ref{sec:exp1}.

\section{Theoretical Results}\label{sec:theo}
\subsection{Upper bound of Failure Probability}
Based on Theorem~\ref{the:1}, which is the theoretical guarantee of the APT algorithm, we derive a theoretical guarantee of TDShap.
The reward distributions must be $R$-sub-Gaussian in Theorem~\ref{the:1}.
On the other hand, on the TDShap algorithm, $R$ can be expressed explicitly by the maximum value of $\Phi_n(\sigma)$ and the minimum value of $\Phi_n(\sigma)$ since there exists only a finite number of possible values of $\Phi_n(\sigma)$.
As a result, the following theorem holds.

\begin{thm}\label{the:2}
Assume that
\begin{equation}\label{equ:katei}
T\geq \dfrac{64Nw^2}{\varepsilon^2}\log(N(1+\log T)),
\end{equation}
where $w\coloneqq\max_{n}\{\max_{\sigma\in S_D}\Phi_n(\sigma)-\min_{\sigma\in S_D}\Phi_n(\sigma)\}$.
For the TDShap algorithm, it holds that
\begin{equation}
p_{\mathrm{fail}}\leq \dfrac{1}{N^2(1+\log T)^2}.\label{eq:fail}
\end{equation}
\end{thm}

\begin{proof}[Proof of Theorem~\ref{the:2}]
The problem complexity defined in Equation (\ref{eq:complexity}) can be upper bounded as
\begin{equation}\label{eq:H_upperbound}
    H=\sum_{n=1}^N\dfrac{1}{(|\phi_n-\tau|+\varepsilon)^2}\leq\dfrac{N}{\varepsilon^2},
\end{equation}
and thus the assumption in Equation (\ref{equ:katei}) is further lower bounded as
\begin{align}
    T&\geq \dfrac{64Nw^2}{\varepsilon^2}\log(N(1+\log T))\nonumber\\
    &\geq 256H\left(\dfrac{w}{2}\right)^2\log(N(1+\log T)).\label{eq:T_lowerbound}
\end{align}
Equation (\ref{eq:T_lowerbound}) is exactly the assumption of Theorem \ref{the:1} for $w/2$-sub-Gaussian.
Here, for any $n$, the value $\Phi_n(\sigma)$ is in the interval $[\min_{\sigma}\Phi_n(\sigma),\max_{\sigma}\Phi_n(\sigma)]$ whose width is at most $w$.
Since the distribution within the width $w$ interval is $w/2$-sub-Gaussian~\citep{subgauss1,subgauss2}, the reward distribution for each instance is $w/2$-sub-Gaussian.
Therefore, Theorem \ref{the:1} with Equation (\ref{eq:H_upperbound}) guarantee the following under the assumption of Equation (\ref{equ:katei}):
\begin{equation*}
    p_{\mathrm{fail}}\leq\exp\left(-\dfrac{T}{128(w/2)^2H}\right)\leq\exp\left(-\dfrac{T\varepsilon^2}{32Nw^2}\right).
\end{equation*}

In addition, from Equation (\ref{equ:katei}) the following holds: 
\begin{equation*}
-\dfrac{T\varepsilon^2}{32Nw^2}\leq -2\log(N(1+\log T)).
\end{equation*}
Hence under the assumption of Equation (\ref{equ:katei}), the following holds:
\begin{equation*}
p_{\mathrm{fail}}\leq\exp\left(-\dfrac{T\varepsilon^2}{32Nw^2}\right)\leq\dfrac{1}{N^2(1+\log T)^2}.
\end{equation*}
\end{proof}

A similar theorem still holds when TDShap is used with enhancement~\ref{enh1}.
Let $S'_{D,n}\coloneqq\{\sigma\in S_D \mid |D'(\sigma,n)|\geq N_{\min}\}$ be the set of permutations of $S_D$ with the cardinality restriction in enhancement~\ref{enh1}, and $w'\coloneqq\max_{n}\{\max_{\sigma\in S'_{D,n}}\Phi_n(\sigma)-\min_{\sigma\in S'_{D,n}}\Phi_n(\sigma)\}$ be the corresponding width of the range of $\Phi_n\left(\sigma\right)$.
By replacing $w$ with $w'$ in the assumption of Equation \ref{equ:katei}, Equation \ref{eq:fail} holds.

\subsection{Upper Bound of Width}\label{sec:wbound}
From Theorem \ref{the:2}, smaller $w$ results in fewer iterations $T$ required.
The upper bound of the width $w$ can be computed in many cases. 
For example, assuming classification problems and $V$ is accuracy, from $0\leq V\leq 1$, it can be derived that $-1\leq\Phi_n(\sigma)\leq 1$ and thus $w\leq 2$.
When assuming regression problems and $V$ is the negative of MSE, under the assumption that the model prediction is between the minimum of label value $y_{\min}$ and the maximum $y_{\max}$ over $D\cup D_{\text{val}}$, it can be derived that $-(y_{\max}-y_{\min})^2\leq V\leq 0$ and $w\leq 2(y_{\max}-y_{\min})^2$.
Similarly, in the case of the negative of MAE, $w\leq 2(y_{\max}-y_{\min})$ holds.
Practically, $w$ is expected to be much smaller than the above upper bounds since adding one instance to the training set has a very small impact on the model in many cases.
Such an assumption is reasonable when the number of instances in the training set is guaranteed to be sufficiently large, for example when using enhancement~\ref{enh1}.
In Appendix~\ref{sec:apeboundtree}, we take the example case of decision trees and show that $w$ can be further upper-bounded under some reasonable assumptions.

\section{Experiments}\label{sec:exp}
\subsection{Settings}
\begin{table*}[t]
    \centering
    \caption{Overview of datasets and tasks. }\label{tab:data1}
    \begin{tabular}{@{}lccrrrcc@{}}
      \toprule
      &&&\multicolumn{3}{c}{\bfseries\# of instances}\\\cmidrule(lr){4-6}
      \bfseries Dataset & \bfseries Type&\bfseries \# of features & \multicolumn{1}{c}{\textbf{Train}}&\multicolumn{1}{c}{\textbf{Val}}&\multicolumn{1}{c}{\textbf{Test}} & \bfseries Task  & \bfseries Model\\\midrule
      Abalone &Table &8 & 1,000&1,000&1,000 & Regression  & Five models\\
      Breast Cancer &Table& 30 & 150 & 150 & 269 & Binary classification  &Decision tree\\
      Fashion-MNIST &Image & $28\times 28$ pixels & 1,000&1,000&1,000 & 10-class classification  &CNN\\
      Livedoor News Corpus&Text &19--11,632 chars & 250&250&500 & 9-class classification  &BERT\\\bottomrule
    \end{tabular}
\end{table*}
We evaluate the effectiveness of the proposed method at data cleansing over various datasets and tasks, which are summarized in Table~\ref{tab:data1}.
We used four datasets for the evaluation: Abalone~\citep{abalone} for tabular regression, Breast Cancer~\citep{breastcancer} for tabular classification, Fashion-MNIST~\citep{fmnist} for image classification, and Livedoor News Corpus\footnote{\url{https://www.rondhuit.com/download_en.html\#ldcc}} for text classification.
Each dataset is split into training, validation, and test sets.

The proposed method is compared with two conventional data cleansing methods: naive data Shapley (Truncated Monte Carlo Shapley in \citet{ds}) and leave-one-out (LOO) \citep{loo}.
In LOO, harmful instances are sequentially identified $n_{\text{LOO}}$ at a time on the basis of the following contribution of the instance $n$:
\begin{equation*}
    \phi_{n}^{\text{LOO}}=V(D)-V(D-\{d_n\}).
\end{equation*}
That is, $n_{\text{LOO}}$ instances with the lowest $\phi_{n}^{\text{LOO}}$ are identified as the most harmful instances, and $n_{\text{LOO}}$ instances with the lowest $\phi_{n}^{\text{LOO}}$ computed after removing the first group from $D$ are identified as the second harmful instances, and so on for the remaining instances.
For each of the three methods, we calculate the contribution of each instance in the training set to the metric $V$ calculated by using the validation set and then remove $n_{\mathrm{remove}}$ instances with low contributions from the training set.
The number of instances $n_{\mathrm{remove}}$ to be removed is determined to maximize the performance of the validation set.
Finally, we evaluate the degree of improvements on the test set by removing $n_{\mathrm{remove}}$ instances from the training set.
For each experiment, we report the average improvements and computation time over 10 trials with different random seeds and dataset split.

We consider that instances with positive (negative) data Shapley values positively (negatively) affect the model performance.
Simply setting $\tau=0$ is possible, but the prediction $\phi_n$ is uncertain when $\phi_n\in\left[-\varepsilon,\varepsilon\right]$; to exclude only those instances that predicted $\phi_n$ to be negative with sufficient certainty, $\tau$ is set to a negative value with a small absolute value and $\varepsilon$ is set to satisfy $\tau+\varepsilon=0$.
More specifically, we set $(\tau,\varepsilon)=(-0.1,0.1)$ in Section \ref{sec:exp1} and $(\tau,\varepsilon)=(-0.01,0.01)$ in Section \ref{sec:exp2}.
As a stopping criterion in line~\ref{algline:while}, we fixed the number of iterations $n_{\text{iter}}$ as $n_{\text{iter}}=500$ for the Fashion-MNIST dataset and $n_{\text{iter}}=50$ for the other datasets.

\subsection{Evaluation using Various Models}\label{sec:exp1}
We first validated that the proposed method is effective independent of the model using the Abalone dataset.
We used decision trees (DTree), support vector regression (SVR), ridge regression, multi-layer perceptron (MLP), and gradient-boosted decision trees (GBDT) as regression models.
We used the scikit-learn implementation~\citep{sklearn} for the first four models and the LightGBM implementation \citep{lgbm} for GBDT.
For baseline LOO, we used $n_\text{LOO}=100$ and $n_\text{LOO}=1$, each of which is referred to as $\text{LOO}_{100}$ and $\text{LOO}_{1}$, respectively.
In addition to LOO and data Shapley, we also report the results of random removal, in which instances to be removed are randomly selected regardless of their contribution to model performance.
Unless otherwise specified, enhancements \ref{enh1} and \ref{enh2} are used for the proposed TDShap in this subsection.
Note that enhancement~\ref{enh3} is not used here since no considered models are based on deep learning.
$N_\mathrm{min}$ in enhancement~\ref{enh1} is set to 900 for SVR and ridge regression and 100 otherwise.
$K$ in enhancement~\ref{enh2} is set to 100.

\paragraph{Prediction performance}
\begin{table*}[t]
    \centering
    \caption{MAEs of various regression models on the Abalone dataset with and without instance removal. The values in each cell are the mean and standard deviation over 10 trials. The result of $\text{LOO}_{1}$ with MLP is not available because the computation time is too long for practical use.
    }\label{tab:resultacc}
    \resizebox{\linewidth}{!}{%
        \begin{tabular}{@{}lcccccc@{}}
            \toprule
            &\textbf{w/o instance removal} &\multicolumn{5}{c}{\bfseries w/ instance removal} \\
            \cmidrule(lr){2-2}\cmidrule(l){3-7}
            \cmidrule(l){4-5}
            \bfseries Model &\bfseries Baseline & \bfseries Random &  {\textbf{LOO}$_{100}$} & {\textbf{LOO}$_{1}$} &\bfseries Data Shapley &\bfseries TDShap \\\midrule
            DTree & $1.739\pm 0.052$ & $1.740\pm 0.055$ &  $1.689\pm 0.053$ & $\textbf{1.678}\pm 0.031$& $1.684\pm 0.047$ & $\textbf{1.678}\pm 0.046$ \\
            SVR & $1.626\pm 0.051$ & $1.625\pm 0.051$  & $1.596\pm 0.046$ & $\textbf{1.577}\pm 0.028$ & $1.595\pm 0.042$ & $1.595\pm 0.041$ \\
            Ridge & $1.617\pm 0.055$ & $1.616\pm 0.056$ &   $1.567\pm 0.032$ & $1.563\pm 0.033$&$1.568\pm 0.049$ & $\textbf{1.562}\pm 0.048$ \\
            MLP & $1.708\pm 0.066$ & $1.687\pm 0.065$ & $1.615\pm 0.053$ & N/A  & $\textbf{1.604}\pm 0.055$ &$1.610\pm 0.040$  \\
            GBDT & $1.645\pm 0.043$ & $1.654\pm 0.042$ &  $1.627\pm 0.033$ & $1.619\pm 0.030$ & $\textbf{1.595}\pm 0.050$ & $\textbf{1.595}\pm 0.041$ \\
            \bottomrule
        \end{tabular}%
    }
\end{table*}

Table~\ref{tab:resultacc} shows MAEs on the Abalone dataset with and without data cleansing.
Baseline represents the results of the model trained by using the entire training set $D$.
While random removal did not improve performance over the baseline, LOO, data Shapley, and the proposed TDShap algorithm consistently reduced MAEs for all five models.
In terms of the range of improvements, TDShap showed almost the same MAE as LOO$_{1}$ and data Shapley with low computational cost except that LOO performs better for SVR and worse for GBDT. 
Especially, in the case of DTree and ridge regression, TDShap achieved better MAEs than data Shapley.
The reason for this appears in Figure~\ref{fig:abalonetree}, which visualizes the MAEs on the validation and test sets when varying the number of removed instances to train DTree.
The gaps between MAEs on the validation and those on test sets are larger in data Shapley than in other methods.
This indicates that data Shapley tends to overfit the validation set, which makes it difficult to determine the best number of instances to be removed.
On the other hand, TDShap does not overfit the validation set.
This may be because TDShap refers to the validation set much fewer times than data Shapley.

\begin{figure}[t]
  \centering
  \subfloat[Validation set]{\includegraphics[width=0.495\linewidth]{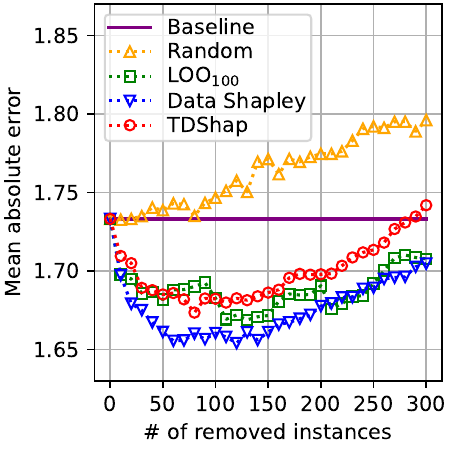}}\hfill
  \subfloat[Test set]{\includegraphics[width=0.495\linewidth]{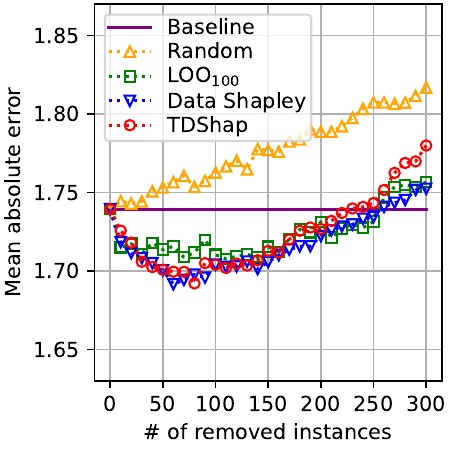}}
  \caption{MAEs of DTree when varying the number of instances to be removed on the Abalone dataset.
}\label{fig:abalonetree}
\end{figure}

\paragraph{Computation time}
\begin{table}[t]
    \centering
    \setlength{\tabcolsep}{1pt}
    \caption{Comparison of computation time between baselines and TDShap on the Abalone dataset.}\label{tab:resultalgotime2}
    \resizebox{0.96\linewidth}{!}{%
    \begin{tabular}{@{}lS[table-format=+1.1e+1]S[table-format=+1.1e+1]S[table-format=+1.1e+1]S[table-format=+1.1e+1]@{}}
        \toprule
        & \multicolumn{4}{c}{\textbf{Computation time [s]}} \\\cmidrule(l){2-5}
         \textbf{Model} &  \textbf{LOO}$_{100}$& \textbf{LOO}$_{1}$ & \textbf{Data Shapley} & \textbf{TDShap}  \\\midrule
        DTree& 2.1e+1 & 3.4e+3 & 2.0e+2 & 9.0 \\
        SVR& 3.5e+2 & 1.6e+4 & 1.2e+3 & 8.4e+1\\
        Ridge& 1.7e+1 & 3.1e+3 & 8.8e+1& 9.3e+0\\
        MLP& 6.2e+3 & {> 1 week} & 4.8e+4& 4.8e+3 \\
        GBDT& 2.6e+2 & 2.3e+4 & 1.0e+3 & 1.0e+2\\\bottomrule
    \end{tabular}%
    }
\end{table}

Table~\ref{tab:resultalgotime2} shows the average computation time of each method.
The proposed TDShap is at least $200$ times faster for all models than LOO$_{1}$.
Also, TDShap is roughly at least ten times faster for all models than the conventional data Shapley with almost the same performance.
This indicated that TDShap could have avoided unnecessary calculations such as ranking within harmful instances and ranking within non-harmful instances.
Figure~\ref{fig:resultalgo} shows the relationship between the number of removed instances and MAE, which also supports this consideration.
In any model, data Shapley and TDShap have almost the same minimum MAE\footnote{The minimum MAEs in Figure~\ref{fig:resultalgo} differ from the values in Table~\ref{tab:resultacc} because the optimal $n_{\text{remove}}$ varies for each of 10 trials.}, but they differ in the range of the number of removed instances where the MAE is minimum, especially in the case of MLP.
More specifically, data Shapley's graph has a flat bottom and the MAE is about minimum in the range of 100 to 250 instances removed; on the other hand, TDShap's graph has a sharp bottom and the MAE is minimum in the range of 200 to 250 instances removed.
This means that TDShap did not correctly rank the 200 instances from the lower data Shapley value, but still did not affect the performance of data cleansing as in Table~\ref{tab:resultacc}.

\begin{figure*}[t]
    \centering
    \subfloat[SVR]{\includegraphics[width=0.245\linewidth]{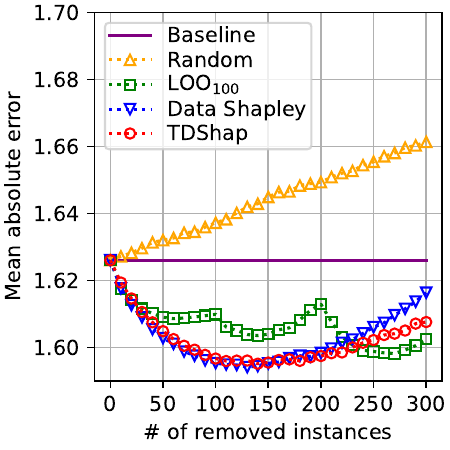}}\hfill
    \subfloat[Ridge]{\includegraphics[width=0.245\linewidth]{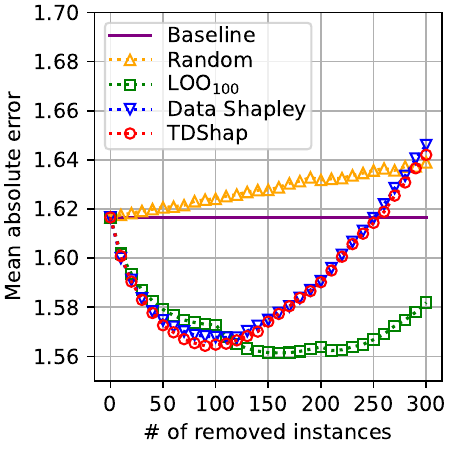}}\hfill
    \subfloat[MLP]{\includegraphics[width=0.245\linewidth]{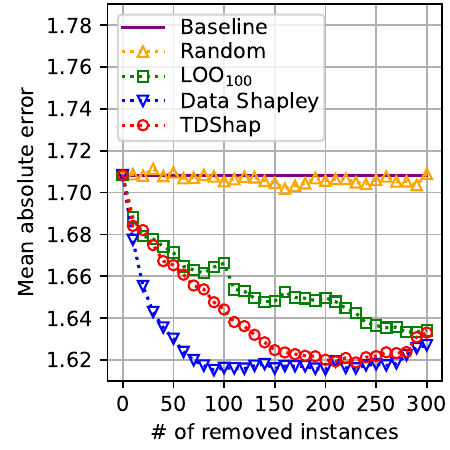}}\hfill
    \subfloat[GBDT]{\includegraphics[width=0.245\linewidth]{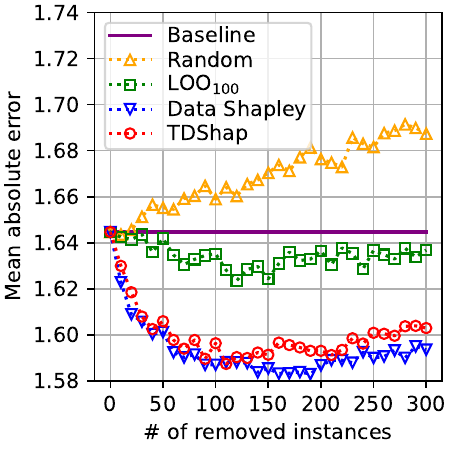}}
    \caption{MAEs when varying the number of instances to be removed on the Abalone dataset (test set).}\label{fig:resultalgo}
\end{figure*}

\paragraph{Ablation study}
Table~\ref{tab:resultenh12} shows the MAEs when enhancements \ref{enh1} and \ref{enh2} are excluded from the proposed method.
Note that enhancement~\ref{enh1} aims to improve accuracy by costing computational time, and enhancement~\ref{enh2} aims to reduce computation time by allowing biased sampling.
Both involve a trade-off between MAE and computation time, but here we compare the MAEs when the computation times are aligned by using the TDShap computation time as the standard; we terminated the computation at the TDShap computation time when either enhancement~\ref{enh1} or \ref{enh2} is not used.
The results in Table~\ref{tab:resultenh12} clearly show that both enhancements improved the MAEs.
\begin{table}[t]
    \centering
    \caption{Ablation study of MAEs on the Abalone dataset with respect to enhancements \ref{enh1} and \ref{enh2}. Note that the calculation time is the same between methods.}\label{tab:resultenh12}
    \resizebox{0.96\linewidth}{!}{%
    \begin{tabular}{@{}lccc@{}}
        \toprule
        \bfseries Model &\bfseries TDShap & \bfseries w/o enh. \ref{enh1} &\bfseries w/o enh. \ref{enh2} \\\midrule
        DTree & $\textbf{1.678}\pm 0.046$ & $1.697\pm 0.051$ & $1.696\pm 0.039$\\
        SVR & $\textbf{1.595}\pm 0.041$ & $1.602\pm 0.054$ & $1.596\pm 0.047$\\
        Ridge & $\textbf{1.562}\pm 0.048$ & $1.571\pm 0.052$ & $1.571\pm 0.050$\\
        MLP & $\textbf{1.610}\pm 0.040$ & $1.628\pm 0.055$ & $1.618\pm 0.049$\\
        GBDT & $\textbf{1.595}\pm 0.041$ & $1.611\pm 0.034$ & $1.604\pm 0.028$ \\
        \bottomrule
    \end{tabular}%
    }
\end{table}

\subsection{Evaluation on Various Datasets}\label{sec:exp2}
We next validated the effectiveness of TDShap independent of the datasets in the classification tasks using three datasets: Breast Cancer, Fashion-MNIST, and Livedoor News Corpus.
We used a decision tree, convolutional neural network (CNN) detailed in Appendix~\ref{sec:apemodel}, and bidirectional encoder representations from Transformers (BERT) \citep{bert} as a classifier for each dataset, respectively.
The same as in Section \ref{sec:exp1}, three methods of LOO with $n_\text{LOO}=|D|$, data Shapley, and TDShap are compared.
Since data Shapley requires computation time that is too long for Fashion-MNIST and Livedoor News Corpus datasets, we also tried terminating the calculation of data Shapley with TDShap computation time.
We denote this method data Shapley (short).
Otherwise specified, enhancements \ref{enh1} and \ref{enh2} are used for the decision tree, and all the enhancements are used for CNN and BERT.
$N_\mathrm{min}$ in enhancement~\ref{enh1} is set to 100, and $K$ in enhancement~\ref{enh2} is set to 50.

\begin{table*}[t]
    \centering
    \caption{Accuracies (\%) with and without instance removal. The results are not available for data Shapley with Fashion-MNIST and Livedoor News Corpus because the computation cannot be completed in a realistic amount of time.}\label{tab:resultaccdata}
    \resizebox{\linewidth}{!}{%
        \begin{tabular}{@{}lccccc@{}}
            \toprule
            &\textbf{w/o instance removal} &\multicolumn{4}{c}{\bfseries w/ instance removal} \\\cmidrule(lr){2-2}\cmidrule(l){3-6}
            \bfseries Dataset &\bfseries Baseline &\bfseries LOO &\bfseries Data Shapley (short) &\bfseries Data Shapley &\bfseries TDShap \\\midrule
            Breast Cancer & $0.903 \pm 0.023$ & $0.919\pm 0.025$ & $0.916\pm 0.024$ & $\textbf{0.929}\pm 0.018$ & $\textbf{0.929}\pm 0.012$ \\
            Fashion-MNIST & $0.825\pm 0.010$ & $0.829\pm 0.007$ & $0.829\pm 0.006$ & N/A &$\textbf{0.832}\pm 0.007$ \\
            Livedoor News Corpus & $0.732\pm 0.016$ & $0.736\pm 0.019$ & $0.734\pm 0.009$ & N/A 
     &$\textbf{0.742}\pm 0.018$ \\
            \bottomrule
        \end{tabular}%
    }
\end{table*}

\paragraph{Prediction performance}
Table~\ref{tab:resultaccdata} shows the accuracies achieved by each method.
All the methods improved the accuracy of the baseline model on all the datasets, and TDShap was the most effective among them.
In the case of Breast Cancer, TDShap and data Shapley achieved almost the same accuracy, but TDShap was $9.79$ times faster than data Shapley in terms of computation time.
In summary, the proposed TDShap has been proven to conduct fast and effective data cleansing for various types of data such as tables, images, and text.

\paragraph{Ablation study}
To validate enhancement~\ref{enh3}, an ablation study was also conducted in this case.
Since enhancement~\ref{enh3} is proposed for deep learning models, we chose Fashion-MNIST and Livedoor News Corpus as evaluation datasets.
When enhancement~\ref{enh3} is not used, the model is initialized with random weights instead of pre-trained weights.
Two types of cases were considered in which enhancement~\ref{enh3} is not used: the case where training is terminated at the same number of epochs as TDShap (\textit{Few-epoch}), and the case where training is continued until convergence (\textit{Full-train}).
As shown in Table~\ref{tab:resultenh3}, enhancement~\ref{enh3} successfully reduced the training duration without compromising classification performance.

\begin{table}[t]
    \centering
    \caption{Ablation study with respect to enhancement~\ref{enh3}. Each cell shows classification accuracy (top) and computational time (bottom).}\label{tab:resultenh3}
    \resizebox{\linewidth}{!}{%
    \begin{tabular}{@{}lccc@{}}
        \toprule
        & \textbf{w/ enh. \ref{enh3}} & \multicolumn{2}{c}{\textbf{w/o enh. \ref{enh3}}}\\
        \cmidrule(lr){2-2}\cmidrule(l){3-4}
        \bfseries Dataset &\bfseries TDShap & \bfseries Few-epoch & \bfseries Full-train \\\midrule
        Fashion-MNIST & $0.832\pm 0.007$ & $0.830\pm 0.005$ & $0.831\pm 0.009$ \\
        & 8.1 hours & 8.1 hours & 42 hours\\
        \midrule
        Livedoor News Corpus& $0.742\pm 0.018$ & $0.731\pm 0.020$ & $0.743\pm 0.019$ \\
        & 2.5 hours & 2.5 hours & 7.2 hours\\
        \bottomrule
    \end{tabular}%
    }
\end{table}

\section{Related Studies}\label{sec:stu}

\paragraph{Contribution measure of instance}
The proposed method identified harmful instances on the basis of their contribution measured by data Shapley~\citep{ds}.
Many other methods have been proposed to evaluate the contribution of instances to the model.
As a model-independent method, there is LOO~\citep{loo}, which was used as one of the baselines in the experiments in this paper.
As model-dependent methods, methods for differentiable models~\citep{if,sgdif,ifgan,tracin}, for decision trees and decision tree ensemble models~\citep{lfif1,lfif2}, and for specific multilayer models~\citep{rep} have been proposed.
In the context of data cleansing, however, methods other than data Shapley do not work well because they do not consider removing multiple instances simultaneously~\citep{ifsv}.

\paragraph{Outlier detection}
For data cleansing, outliers in a training set are naturally considered as harmful, and indeed a lot of outlier detection methods have been proposed for this purpose including one-class SVM~\citep{osvm}, local outlier factor~\citep{lof}, isolation forest~\citep{isof}, and others~\citep{od1,od2}.
On the other hand, outliers are not necessarily harmful to model performance~\citep{sgdif}.

\paragraph{Neuron Shapley}
Neuron Shapley aims to identify the most influential neurons in a deep learning model~\citep{ns}.
This is similar to the proposed TDShap in that it utilizes both the Shapley value and the multi-armed bandit algorithm, but the purpose of adopting them and the algorithms used are different.
While TDShap aims to identify instances with data Shapley values below a threshold using threshold bandits, Neuron Shapley aims to quickly identify the neurons with the top-$K$ Shapley values by using top-$K$ bandits \citep{topk1,topk2,topk3}.

\section{Conclusion and Future Work}\label{sec:con}

This paper proposed a fast data cleansing method named TDShap.
It identifies instances whose data Shapley values are below a given threshold using a thresholding bandit algorithm.
For faster computation and better prediction performance, we propose three enhancements: restriction on training subsets, evaluation of multiple instances, and pre-training.
We provided both theoretical guarantees on TDShap and empirical results using various models and datasets.

One future direction of this work is to apply a bandit framework other than the APT algorithm for data cleansing.
For example, though the proposed method aimed to identify all instances with lower contributions than the predetermined threshold, it may be preferable to identify a specified number of instances in order of contribution in some cases.
In such cases, the top-$K$ bandit framework can be a good alternative.
Also, when multiple instances are evaluated simultaneously as in enhancement~\ref{enh3}, the combinatorial bandit framework~\citep{comb} can be utilized.

% References
\bibliography{template}

\begin{thebibliography}{39}
\providecommand{\natexlab}[1]{#1}
\providecommand{\url}[1]{\texttt{#1}}
\expandafter\ifx\csname urlstyle\endcsname\relax
  \providecommand{\doi}[1]{doi: #1}\else
  \providecommand{\doi}{doi: \begingroup \urlstyle{rm}\Url}\fi

\bibitem[Boukerche et~al.(2020)Boukerche, Zheng, and Alfandi]{od1}
Azzedine Boukerche, Lining Zheng, and Omar Alfandi.
\newblock Outlier detection: Methods, models, and classification.
\newblock \emph{ACM Computing Surveys}, 53\penalty0 (3):\penalty0 1--37, 2020.

\bibitem[Breunig et~al.(2000)Breunig, Kriegel, Ng~Raymond, and Sander]{lof}
Markus~M. Breunig, Hans-Peter Kriegel, T.~Ng~Raymond, and J{\"o}rg Sander.
\newblock {LOF}: identifying density-based local outliers.
\newblock In \emph{ACM SIGMOD International Conference on Management of Data}, pages 93--104, 2000.

\bibitem[Brophy et~al.(2023)Brophy, Hammoudeh, and Lowd]{lfif2}
Jonathan Brophy, Zayd Hammoudeh, and Daniel Lowd.
\newblock Adapting and evaluating influence-estimation methods for gradient-boosted decision trees.
\newblock \emph{Journal of Machine Learning Research}, 24\penalty0 (154):\penalty0 1--48, 2023.

\bibitem[Bubeck et~al.(2013)Bubeck, Wang, and Viswanathan]{topk1}
S{\'e}ebastian Bubeck, Tengyao Wang, and Nitin Viswanathan.
\newblock Multiple identifications in multi-armed bandits.
\newblock In \emph{International Conference on Machine Learning}, pages 258--265. PMLR, 2013.

\bibitem[Chalkidis et~al.(2019)Chalkidis, Fergadiotis, Malakasiotis, and Androutsopoulos]{eurlex57k}
Ilias Chalkidis, Emmanouil Fergadiotis, Prodromos Malakasiotis, and Ion Androutsopoulos.
\newblock Large-scale multi-label text classification on {EU} legislation.
\newblock In \emph{Annual Meeting of the Association for Computational Linguistics}, pages 6314--6322, 2019.

\bibitem[Chen et~al.(2021)Chen, Chai, Wang, Du, Zhang, Weng, Su, Povey, Trmal, Zhang, Jin, Khudanpur, Watanabe, Zhao, Zou, Li, Yao, Wang, Wang, and Zhao~You]{gigaspeech}
Guoguo Chen, Shuzhou Chai, Guanbo Wang, Jiayu Du, Wei-Qiang Zhang, Chao Weng, Dan Su, Daniel Povey, Jan Trmal, Junbo Zhang, Mingjie Jin, Sanjeev Khudanpur, Shinji Watanabe, Shuaijiang Zhao, Wei Zou, Xiangang Li, Xuchen Yao, Yongqing Wang, Yujun Wang, and Zhiyong~Yan Zhao~You.
\newblock {GigaSpeech}: An evolving, multi-domain {ASR} corpus with 10,000 hours of transcribed audio.
\newblock In \emph{The Annual Conference of the International Speech Communication Association (INTERSPEECH)}, pages 3670--3674, 2021.

\bibitem[Chen et~al.(2013)Chen, Wang, and Yuan]{comb}
Wei Chen, Yajun Wang, and Yang Yuan.
\newblock Combinatorial multi-armed bandit: General framework and applications.
\newblock In \emph{International Conference on Machine Learning}, pages 151--159. PMLR, 2013.

\bibitem[Chih-Kuan et~al.(2018)Chih-Kuan, Joon, En-Hsu, and K]{rep}
Yeh Chih-Kuan, Kim Joon, Yen~Ian En-Hsu, and Ravikumar~Pradeep K.
\newblock Representer point selection for explaining deep neural networks.
\newblock In \emph{Advances in Neural Information Processing Systems}, volume~31, pages 9291--9301, 2018.

\bibitem[Cook and Sanford(1982)]{loo}
R.~Dennis Cook and Weisberg Sanford.
\newblock Criticism and influence analysis in regression.
\newblock \emph{Sociological Methodology}, 13:\penalty0 313--361, 1982.

\bibitem[Deng et~al.(2009)Deng, Dong, Socher, Li, Li, and Fei-Fei]{imagenet}
Jia Deng, Wei Dong, Richard Socher, Li-Jia Li, Kai Li, and Li~Fei-Fei.
\newblock {ImageNet}: A large-scale hierarchical image database.
\newblock In \emph{IEEE Conference on Computer Vision and Pattern Recognition}, pages 248--255, 2009.

\bibitem[Devlin et~al.(2019)Devlin, Chang, Lee, and Toutanova]{bert}
Jacob Devlin, Ming-Wei Chang, Kenton Lee, and Kristina Toutanova.
\newblock {BERT}: Pre-training of deep bidirectional transformers for language understanding.
\newblock In \emph{Annual Conference of the North American Chapter of the Association for Computational Linguistics: Human Language Technologies}, pages 4171--4186, 2019.

\bibitem[Gabillon et~al.(2012)Gabillon, Ghavamzadeh, and Lazaric]{topk2}
Victor Gabillon, Mohammad Ghavamzadeh, and Alessandro Lazaric.
\newblock Best arm identification: A unified approach to fixed budget and fixed confidence.
\newblock In \emph{Advances in Neural Information Processing Systems}, volume~25, pages 3212--3220, 2012.

\bibitem[Ghorbani and Zou(2019)]{ds}
Amirata Ghorbani and James Zou.
\newblock Data {Shapley}: Equitable valuation of data for machine learning.
\newblock In \emph{International Conference on Machine Learning}, pages 2242--2251. PMLR, 2019.

\bibitem[Ghorbani and Zou(2020)]{ns}
Amirata Ghorbani and James Zou.
\newblock Neuron {Shapley}: Discovering the responsible neurons.
\newblock In \emph{Advances in Neural Information Processing Systems}, pages 5922--5932, 2020.

\bibitem[Hammoudeh and Lowd(2022)]{ifsv}
Zayd Hammoudeh and Daniel Lowd.
\newblock Training data influence analysis and estimation: A survey.
\newblock \emph{arXiv preprint arXiv:2212.04612}, 2022.

\bibitem[Hara et~al.(2019)Hara, Nitanda, and Maehara]{sgdif}
Satoshi Hara, Atsushi Nitanda, and Takanori Maehara.
\newblock Data cleansing for models trained with {SGD}.
\newblock In \emph{Advances in Neural Information Processing Systems}, volume~32, pages 4213--4222, 2019.

\bibitem[Hoeffding(1963)]{subgauss1}
Wassily Hoeffding.
\newblock Probability inequalities for sums of bounded random variables.
\newblock \emph{Journal of the American Statistical Association}, 58\penalty0 (301):\penalty0 13--30, 1963.

\bibitem[Jia et~al.(2019)Jia, Dao, Wang, Hubis, G{\"u}rel, Li, Zhang, Spanos, and Song]{knnshap}
Ruoxi Jia, David Dao, Boxin Wang, Frances~A Hubis, Nezihe~M G{\"u}rel, Bo~Li, Ce~Zhang, Costas~J Spanos, and Dawn Song.
\newblock Efficient task-specific data valuation for nearest neighbor algorithms.
\newblock \emph{Proceedings of the VLDB Endowment}, 12\penalty0 (11):\penalty0 1610--1623, 2019.

\bibitem[Kaufmann et~al.(2016)Kaufmann, Capp{\'e}, and Garivier]{topk3}
Emilie Kaufmann, Olivier Capp{\'e}, and Aur{\'e}lien Garivier.
\newblock On the complexity of best arm identification in multi-armed bandit models.
\newblock \emph{Journal of Machine Learning Research}, 17:\penalty0 1--42, 2016.

\bibitem[Ke et~al.(2017)Ke, Meng, Finley, Wang, Chen, Ma, Ye, and Liu]{lgbm}
Guolin Ke, Qi~Meng, Thomas Finley, Taifeng Wang, Wei Chen, Weidong Ma, Qiwei Ye, and Tie-Yan Liu.
\newblock {LightGBM}: A highly efficient gradient boosting decision tree.
\newblock In \emph{Advances in Neural Information Processing Systems}, volume~30, pages 3146--3154, 2017.

\bibitem[Khang et~al.(2023)Khang, Rana, Tailor, and Abdullayev]{dch}
Alex Khang, Geeta Rana, R.K. Tailor, and Vugar Abdullayev.
\newblock \emph{Data-Centric AI Solutions and Emerging Technologies in the Healthcare Ecosystem}.
\newblock CRC Press, 2023.

\bibitem[Koh and Liang(2017)]{if}
Pang~Wei Koh and Percy Liang.
\newblock Understanding black-box predictions via influence functions.
\newblock In \emph{International Conference on Machine Learning}, pages 1885--1894. PMLR, 2017.

\bibitem[Lerasle(2019)]{subgauss2}
Matthieu Lerasle.
\newblock Selected topics on robust statistical learning theory.
\newblock \emph{Lecture Notes}, 2019.

\bibitem[Liu et~al.(2008)Liu, Ting, and Zhou]{isof}
Fei~Tony Liu, Kai~Ming Ting, and Zhi-Hua Zhou.
\newblock Isolation forest.
\newblock In \emph{IEEE International Conference on Data Mining}, pages 413--422. IEEE, 2008.

\bibitem[Locatelli et~al.(2016)Locatelli, Gutzeit, and Carpentier]{apt}
Andrea Locatelli, Maurilio Gutzeit, and Alexandra Carpentier.
\newblock An optimal algorithm for the thresholding bandit problem.
\newblock In \emph{International Conference on Machine Learning}, pages 1690--1698. PMLR, 2016.

\bibitem[Motamedi et~al.(2021)Motamedi, Sakharnykh, and Kaldewey]{dcs}
Mohammad Motamedi, Nikolay Sakharnykh, and Tim Kaldewey.
\newblock A data-centric approach for training deep neural networks with less data.
\newblock \emph{arXiv preprint arXiv:2110.03613}, 2021.

\bibitem[Nash et~al.(1995)Nash, Sellers, Talbot, Cawthorn, and Ford]{abalone}
Warwick Nash, Tracy Sellers, Simon Talbot, Andrew Cawthorn, and Wes Ford.
\newblock Abalone.
\newblock UCI Machine Learning Repository, 1995.

\bibitem[Pedregosa et~al.(2011)Pedregosa, Varoquaux, Gramfort, Michel, Thirion, Grisel, Blondel, Prettenhofer, Weiss, Dubourg, Vanderplas, Passos, Cournapeau, Brucher, Perrot, and Duchesnay]{sklearn}
Fabian Pedregosa, Ga{\"e}l Varoquaux, Alexandre Gramfort, Vincent Michel, Bertrand Thirion, Olivier Grisel, Mathieu Blondel, Peter Prettenhofer, Ron Weiss, Vincent Dubourg, Jake Vanderplas, Alexandre Passos, David Cournapeau, Matthieu Brucher, Matthieu Perrot, and {{\'E}}douard Duchesnay.
\newblock Scikit-learn: Machine learning in {Python}.
\newblock \emph{Journal of Machine Learning Research}, 12:\penalty0 2825--2830, 2011.

\bibitem[Pruthi et~al.(2020)Pruthi, Liu, Kale, and Sundararajan]{tracin}
Garima Pruthi, Frederick Liu, Satyen Kale, and Mukund Sundararajan.
\newblock Estimating training data influence by tracing gradient descent.
\newblock In \emph{Advances in Neural Information Processing Systems}, volume~33, pages 19920--19930, 2020.

\bibitem[Sch{\"o}lkopf et~al.(2001)Sch{\"o}lkopf, Platt, Shawe-Taylor, Smola, and Williamson]{osvm}
Bernhard Sch{\"o}lkopf, John~C Platt, John Shawe-Taylor, Alex~J Smola, and Robert~C Williamson.
\newblock Estimating the support of a high-dimensional distribution.
\newblock \emph{Neural computation}, 13\penalty0 (7):\penalty0 1443--1471, 2001.

\bibitem[Sharchilev et~al.(2018)Sharchilev, Ustinovskiy, Serdyukov, and Rijke]{lfif1}
Boris Sharchilev, Yury Ustinovskiy, Pavel Serdyukov, and Maarten Rijke.
\newblock Finding influential training samples for gradient boosted decision trees.
\newblock In \emph{International Conference on Machine Learning}, pages 4577--4585. PMLR, 2018.

\bibitem[Terashita et~al.(2020)Terashita, Ohashi, Nonaka, and Kanemaru]{ifgan}
Naoyuki Terashita, Hiroki Ohashi, Yuichi Nonaka, and Takashi Kanemaru.
\newblock Influence estimation for generative adversarial networks.
\newblock In \emph{International Conference on Learning Representations}, 2020.

\bibitem[Wang et~al.(2019)Wang, Bah, and Hammad]{od2}
Hongzhi Wang, Mohamed~Jaward Bah, and Mohamed Hammad.
\newblock Progress in outlier detection techniques: A survey.
\newblock \emph{IEEE Access}, 7:\penalty0 107964--108000, 2019.

\bibitem[Wang et~al.(2023)Wang, Zhu, Wang, Jia, and Mittal]{tknnshap}
Jiachen~T Wang, Yuqing Zhu, Yu-Xiang Wang, Ruoxi Jia, and Prateek Mittal.
\newblock Threshold {KNN-Shapley}: A linear-time and privacy-friendly approach to data valuation.
\newblock In \emph{Advances in Neural Information Processing Systems}, volume~36, 2023.

\bibitem[Wang et~al.(2022)Wang, Yang, and Jia]{fastds}
Tianhao Wang, Yu~Yang, and Ruoxi Jia.
\newblock Improving cooperative game theory-based data valuation via data utility learning.
\newblock \emph{International Conference on Learning Representations 2022 Workshop on Socially Responsible Machine Learning}, 2022.

\bibitem[Whang et~al.(2023)Whang, Roh, Song, and Lee]{dcsv}
Steven~Euijong Whang, Yuji Roh, Hwanjun Song, and Jae-Gil Lee.
\newblock Data collection and quality challenges in deep learning: A data-centric {AI} perspective.
\newblock \emph{The VLDB Journal}, 32\penalty0 (4):\penalty0 791--813, 2023.

\bibitem[Wolberg et~al.(1995)Wolberg, Mangasarian, Street, and Street]{breastcancer}
William Wolberg, Olvi Mangasarian, Nick Street, and W.~Street.
\newblock {Breast Cancer Wisconsin (Diagnostic)}.
\newblock UCI Machine Learning Repository, 1995.

\bibitem[Xiao et~al.(2017)Xiao, Rasul, and Vollgraf]{fmnist}
Han Xiao, Kashif Rasul, and Roland Vollgraf.
\newblock {Fashion-MNIST}: a novel image dataset for benchmarking machine learning algorithms.
\newblock \emph{arXiv preprint arXiv:1708.07747}, 2017.

\bibitem[Zha et~al.(2023)Zha, Bhat, Lai, Yang, Jiang, Zhong, and Hu]{dcsv2}
Daochen Zha, Zaid~Pervaiz Bhat, Kwei-Herng Lai, Fan Yang, Zhimeng Jiang, Shaochen Zhong, and Xia Hu.
\newblock Data-centric artificial intelligence: A survey.
\newblock \emph{arXiv preprint arXiv:2303.10158}, 2023.

\end{thebibliography}

\newpage

\onecolumn

\title{Thresholding Data Shapley for Data Cleansing Using Multi-Armed Bandits\\(Supplementary Material)}
\maketitle

\appendix
\section{Tighter Upper Bound of Width for Decision Trees}\label{sec:apeboundtree}
In Section~\ref{sec:wbound}, we provided the upper bound of $w$ independent of types of models when $V$ is negative MSE as
\begin{equation}
    w'\leq 2(y_{\mathrm{max}}-y_{\mathrm{min}})^2.\label{eq:wbound}
\end{equation}
In this section, we provide a tighter upper bound of $w$ under realistic assumptions for the case where the model is a decision tree.

\begin{prop}\label{the:wbound}
Let $\mathcal{D}'\coloneqq \{D'\subset D\mid |D'|\geq N_{\min}\}$ be the set of all possible training subsets for TDShap with enhancement~\ref{enh1}.
Consider the case when the model is a decision tree and the performance metric $V$ is the negative of MSE.
Assume that for any $D'\in \mathcal{D}'$, each leaf of the decision tree trained using $D'$ has $n_{\mathrm{instance}}$ or more training instances.
Assume that the structure of the decision tree trained using $D'\in \mathcal{D}'$ does not depend on $D'$ except for the prediction values of each leaf.
Then, it holds that
\begin{equation*}
w'\leq 2(y_{\mathrm{max}}-y_{\mathrm{min}})^2\times\dfrac{2}{n_{\mathrm{instance}}+1}.
\end{equation*}
\end{prop}
The result of Proposition~\ref{the:wbound} is tighter in terms of the factor $2/(n_{\mathrm{instance}}+1)$ than the following general results for MSE in Equation~\ref{eq:wbound}.
The first assumption on $n_{\mathrm{instance}}$ of Proposition~\ref{the:wbound} is reasonable in the sense that $n_{\mathrm{instance}}$ can usually be given as a hyperparameter, e.g., \texttt{min\_samples\_leaf} in the scikit-learn implementation~\citep{sklearn}.
The second assumption that the tree structure does not change is extreme as a general statement, but it holds when $N_{\min}$ is close to $|D|$.
Especially in the case where $N_{\min}=|D|-1$, it turns into the assumption that the tree structure is invariant by removing one instance from the training set, and methods based on this assumption are known to work well~\citep{lfif1,lfif2}.

\begin{proof}[Proof of Proposition~\ref{the:wbound}]
Since $w'\leq 2\max_{n,\sigma\in S'_{D,n}}|\Phi_n(\sigma)|$ from the definition of $w'$, it is sufficient to prove
\begin{equation*}
|\Phi_n(\sigma)|\leq\dfrac{2(y_{\max}-y_{\min})^2}{n_{\mathrm{instance}}+1}
\end{equation*}
for any $n$ with $d_n\in D$ and $\sigma\in S'_{D,n}=\{\sigma\in S_D \mid |D'(\sigma,n)|\geq N_{\min}\}$. 

Let $l$ be the leaf to which $d_n$ belongs of the decision tree trained using $D'(\sigma,n)$.
The instance $d_n$ whose label value is $y_n$ also belongs to $l$ in the decision tree trained using $D'(\sigma,n)\cup\{d_n\}$ under the assumption of unchanging structure.
Let $y^{\mathrm{val}}_1,\dots,y^{\mathrm{val}}_{|D_{\text{val}}|}$ be the label value of each instance of the validation set $D_{\text{val}}$.
Let $y_a,y_b$ be the prediction value of $l$ trained using $D'(\sigma,n)$ and $D'(\sigma,n)\cup\{d_n\}$ respectively.

Since the metric $V$ is the negative of MSE, it holds that:
\begin{align}\label{eq:evalphi}
|\Phi_n(\sigma)|
&=\left|V(D'(\sigma,n)\cup\{d_n\})-V(D'(\sigma,n))\right|\notag \\
&\leq\dfrac{1}{|D_{\text{val}}|}\left|\displaystyle\sum_{i=1}^{|D_{\text{val}}|}(y^{\mathrm{val}}_i-y_a)^2-\sum_{i=1}^{|D_{\text{val}}|}(y^{\mathrm{val}}_i-y_b)^2\right|\notag \\
&=\dfrac{1}{|D_{\text{val}}|}\left|\displaystyle\sum_{i=1}^{|D_{\text{val}}|}\left\{-2(y_a-y_b)y^{\mathrm{val}}_i+(y_a^2-y_b^2)\right\}\right|\notag \\
&=\dfrac{2|y_a-y_b|}{|D_{\text{val}}|}\displaystyle\sum_{i=1}^{|D_{\text{val}}|}\left|y^{\mathrm{val}}_i-\dfrac{y_a+y_b}{2}\right|\notag \\
&\leq 2|y_a-y_b|(y_{\max}-y_{\min}).
\end{align}

Let $y^{\mathrm{tra}}_1,\dots,y^{\mathrm{tra}}_k$ be the label value of each instance of the training set $D'(\sigma,n)$ that belongs to $l$ of the decision tree.
Since the performance metric $V$ is the negative of MSE, the prediction value is the average of training instances:
\begin{align*}
y_a&=\dfrac{1}{k}\sum_{i=1}^ky^{\mathrm{tra}}_i,\\
y_b&=\dfrac{1}{k+1}y_n+\dfrac{1}{k+1}\sum_{i=1}^ky^{\mathrm{tra}}_i.
\end{align*}
The maximum of $|y_a-y_b|$ is achieved when $y_n=y_{\min}$ and $y^{\mathrm{tra}}_i=y_{\max}\:(i=1,\dots,k)$ or $y_n=y_{\max}$ and $y^{\mathrm{tra}}_i=y_{\min}\:(i=1,\dots,k)$. Hence the following holds:
\begin{equation}\label{eq:evalydiff}
|y_a-y_b|\leq \dfrac{y_{\max}}{k+1}-\dfrac{k\times y_{\min}}{k(k+1)}=\dfrac{y_{\max}-y_{\min}}{k+1}.
\end{equation}
By combining Equations (\ref{eq:evalphi})--(\ref{eq:evalydiff}) and the fact that $k\geq n_{\mathrm{instance}}$ due to the assumption of Proposition~\ref{the:wbound}, it holds that:
\begin{equation*}
|\Phi_n(\sigma)|\leq \dfrac{2(y_{\max}-y_{\min})^2}{k+1}\leq\dfrac{2(y_{\max}-y_{\min})^2}{n_{\mathrm{instance}}+1}.
\end{equation*}
\end{proof}

\section{Experimental Details}\label{sec:apemodel}
\subsection{Details of the experiments in Section \ref{sec:exp1}}
We used the scikit-learn 1.2.0 \citep{sklearn} and LightGBM 4.1.0 \citep{lgbm} implementations for the experiments in Section~\ref{sec:exp2}.
We used the default hyperparameters for almost all models. 
Specifically, each model is initialized by the following code:
\begin{Verbatim}[frame=single]
# Decision tree (DTree)
from sklearn.tree import DecisionTreeRegressor
model = DecisionTreeRegressor(max_depth=5, min_samples_leaf=64)

# Support vector regression (SVR)
from sklearn.svm import SVR
model = SVR()

# Ridge regression
from sklearn.linear_model import Ridge
model = Ridge()

# Multi-layer perceptron (MLP)
from sklearn.neural_network import MLPRegressor
model = MLPRegressor(max_iter=1000, batch_size=1000)

# Gradient-boosted decision tree (GBDT)
import lightgbm as lgb
model = lgb.LGBMRegressor()
\end{Verbatim}

Note that the values for \texttt{max\_depth} and \texttt{min\_samples\_leaf} of \texttt{DecisionTreeRegressor} are optimized in advance for the baseline model to minimize MAE on the validation set.
For \texttt{MLPRegressor}, we used full-batch training instead of mini-batch training; this is because if there is a mini-batch with few instances for some training subset $D'$, the model trained using $D'$ has low prediction performance.

\subsection{Details of the experiments in Section \ref{sec:exp2}}
The decision tree used in Section~\ref{sec:exp2} for the Breast Cancer dataset is initialized as in the following code:
\begin{Verbatim}[frame=single]
from sklearn.tree import DecisionTreeClassifier
model = DecisionTreeClassifier(max_depth=5, min_samples_leaf=2)
\end{Verbatim}

Table~\ref{tab:cnn} shows the structure of the CNN used in Section~\ref{sec:exp2} for the Fashion-MNIST dataset. 
We set the number of epochs for full training to $200$, the number of epochs for additional training for enhancement~\ref{enh3} to $20$, and batch size to $64$.

\begin{table*}[b]
    \centering
    \caption{Detailed structure of CNN used in our experiment.}\label{tab:cnn}
    \begin{tabular}{@{}lccl@{}}
      \toprule
      \bfseries Layer & \bfseries Output shape &\bfseries \# of parameters&Configuration\\
      \midrule
        Input& (1, 28, 28)& -- & $28\times28$ grayscale image\\\midrule
        \multirow{2}{*}{Conv1}& (32, 28, 28)& 320&$3\times3$ conv, 32-ch \& ReLU\\\cmidrule(l){2-4}
        & (32, 14, 14)& -- & $2\times2$ max pooling\\\midrule
        \multirow{2}{*}{Conv2}& (64, 14, 14)& 18.496&$3\times3$ conv, 64-ch \& ReLU\\\cmidrule(l){2-4}
        & (64, 7, 7)& -- & $2\times2$ max pooling\\\midrule
        Linear1& 128& 401,536&128-dim fully-connected, ReLU, dropout\\\midrule
        Linear2& 10& 1,290 &10-dim fully-connected\\
      \bottomrule
    \end{tabular}
\end{table*}

As the BERT model used in Section~\ref{sec:exp2} for the Livedoor News Corpus, we used the pre-trained model downloaded from Hugging Face\footnote{\url{https://huggingface.co/cl-tohoku/bert-base-japanese}}.
We set the number of epochs for training from the pre-trained model using the whole training set $D$ to $15$ with 500 warm-up steps, the number of epochs for additional training using each subset $D'$ to $5$, weight decay to $0.01$, and batch size to $8$.

\end{document}